\documentclass{article}

\usepackage{microtype}
\usepackage{graphicx}
\usepackage{subcaption}
\usepackage{booktabs} 

\usepackage{hyperref}

\usepackage[preprint]{icml2026}

\usepackage{amsmath}
\usepackage{amssymb}
\usepackage{mathtools}
\usepackage{amsthm}

\usepackage[capitalize,noabbrev]{cleveref}

\theoremstyle{plain}
\newtheorem{theorem}{Theorem}[section]

\newtheorem{lemma}[theorem]{Lemma}

\theoremstyle{definition}
\newtheorem{definition}[theorem]{Definition}

\newtheorem{example}[theorem]{Example}
\theoremstyle{remark}
\newtheorem{remark}[theorem]{Remark}

\usepackage[textsize=tiny]{todonotes}

\usepackage{amsmath,amsfonts,bm}

\def\eqref#1{equation~\ref{#1}}

\def\1{\bm{1}}

\DeclareMathAlphabet{\mathsfit}{\encodingdefault}{\sfdefault}{m}{sl}
\SetMathAlphabet{\mathsfit}{bold}{\encodingdefault}{\sfdefault}{bx}{n}

\usepackage{xspace}

\newcommand{\eg}{\textit{e.g.}\@\xspace}

\icmltitlerunning{Deterministic Output in  Charge-Conserving Continuous-Time Neuromorphic Systems}

\begin{document}

\twocolumn[
  \icmltitle{Determinism in the Undetermined: Deterministic Output in  Charge-Conserving Continuous-Time Neuromorphic Systems with Temporal Stochasticity
  }


  \icmlsetsymbol{equal}{*}

  \begin{icmlauthorlist}
    \icmlauthor{Jing Yan}{math}
    \icmlauthor{Kang You}{eiee,sail}
    \icmlauthor{Zhezhi He}{eiee,sail}
    \icmlauthor{Yaoyu Zhang}{math}
  \end{icmlauthorlist}

  \icmlaffiliation{math}{School of Mathematical Sciences, Institute of Natural Sciences and MOE-LSC, Shanghai Jiao Tong University, Shanghai, China}
  \icmlaffiliation{eiee}{School of Electronic Information and Electrical Engineering, Shanghai Jiao Tong University, Shanghai, China}
  \icmlaffiliation{sail}{Shanghai Artificial Intelligence Laboratory, Shanghai, China}

  \icmlcorrespondingauthor{Zhezhi He}{zhezhi.he@sjtu.edu.cn}
  \icmlcorrespondingauthor{Yaoyu Zhang}{zhyy.sjtu@sjtu.edu.cn}

  \icmlkeywords{Machine Learning, ICML}

  \vskip 0.3in
]



\printAffiliationsAndNotice{}  

\begin{abstract}
Achieving deterministic computation results in asynchronous neuromorphic systems remains a fundamental challenge due to the inherent temporal stochasticity of continuous-time hardware. To address this, we develop a unified continuous-time framework for spiking neural networks (SNNs) that couples the Law of Charge Conservation with minimal neuron-level constraints. This integration ensures that the terminal state depends solely on the aggregate input charge, providing a unique cumulated output invariant to temporal stochasticity. We prove that this mapping is strictly invariant to spike timing in acyclic networks, whereas recurrent connectivity can introduce temporal sensitivity. Furthermore, we establish an exact representational correspondence between these charge-conserving SNNs and quantized artificial neural networks, bridging the gap between static deep learning and event-driven dynamics without approximation errors. These results establish a rigorous theoretical basis for designing continuous-time neuromorphic systems that harness the efficiency of asynchronous processing while maintaining algorithmic determinism.
\end{abstract}

\section{Introduction}
Spiking neural networks (SNNs) represent a paradigm shift in computing,
moving away from clocked, synchronous processing toward event-driven,
asynchronous operation typical of biological brains \cite{merolla2014million}.
By processing information through discrete spike events rather than continuous signals,
SNNs offer the potential for massive parallelism and energy efficiency.
However, realizing this potential in physical neuromorphic hardware introduces a fundamental challenge: \emph{temporal stochasticity}.
In realistic continuous-time settings, the precise timing of spike events is inherently variable due to asynchronous input delivery \cite{gallego2020event}, jitter \cite{dally2004principles}, and non-uniform spike propagation delays \cite{li2025deterministic}.
As a result, the same computation may unfold through different spike orders and latencies across runs, potentially changing the final outcome.

This variability raises a critical question for system design:
\textbf{How can we guarantee \emph{algorithmic determinism} in a substrate that is temporally undetermined?}
In standard digital logic, a global clock ensures synchronization. In asynchronous neuromorphic systems, however, small timing variations can propagate through the network and affect the final result.
If the terminal output depends on the relative arrival order of spikes, the system is not reliable for robust engineering. 
Therefore, achieving reliable computation requires identifying structural principles under which \emph{algorithmic determinism} can be realized.

To achieve this, we develop a unified framework that builds on the \textbf{Law of Charge Conservation (LoCC)} \cite{purcell2013electricity,you2025vistream}. We posit that LoCC should not be treated in isolation, but rather as the cornerstone of a broader structural design that explicitly integrates minimal neuron-level constraints. While LoCC ensures that the cumulative charge balance remains invariant to timing, it alone does not guarantee that the resulting computation is unique or even
well-defined. To address this, we introduce two additional neuron-level conditions,
which together ensure a unique and reachable terminal state;
these conditions are formalized in Section~\ref{sec:charge-conserving-snn}.
By treating physical conservation and these decoding constraints as an integral whole, we establish a robust mechanism where algorithmic determinism emerges naturally from continuous-time dynamics, effectively filtering out the noise of temporal stochasticity.

This structural determinism bridges the gap between dynamic spiking systems and static deep learning models. By ensuring that the terminal state is time-independent, we establish an exact representational correspondence between our charge-conserving SNNs and quantized artificial neural networks (QANNs) . 
This correspondence implies that SNNs can approach the inference accuracy of a mature ecosystem of QANNs in the continuous-time domain, while being deployed on event-driven chips without synchronization.

In summary, this work makes three main contributions:
\begin{itemize}
    \item \textbf{A unified continuous-time framework for charge-conserving SNNs.} We formulate a model-independent LoCC for general continuous-time spiking systems and identify the minimal neuron-level design conditions required to couple charge conservation with unique steady-state decoding.
    
    \item \textbf{Theoretical guarantee for deterministic output under temporal stochasticity.} We rigorously prove that SNNs satisfying our framework with acyclic connectivity admit a unique terminal output. We demonstrate that this output depends solely on the aggregate injected charge and is strictly invariant to the temporal realization of spike events, including arbitrary delays and asynchronous processing.
    \item \textbf{Exact representational correspondence between charge-conserving SNNs and QANNs.} We establish a rigorous bidirectional mapping between the steady-state output of acyclic SNNs and acyclic QANNs. We demonstrate that charge-conserving SNNs naturally realize static QANN functions at equilibrium, and conversely, that quantized QANNs can be implemented by event-driven SNNs without approximation errors.
\end{itemize}

\section{Related Works}

Our work relates to several lines of research on SNNs, continuous-time neural dynamics, and conservation- or constraint-based views of neural computation.

\paragraph{Asynchronous and event-driven spiking computation.}
SNNs are inherently event-driven and asynchronous, a property that underlies both their biological plausibility and their appeal for energy-efficient neuromorphic hardware \cite{RN271,boahen2017neuromorph,davies2018loihi,RN285}.
In such systems, spike timing variability arises naturally from non-uniform spike propagation delays and non-uniform propagation delays.
A substantial body of work has investigated how this variability affects learning and inference, often from an empirical or algorithmic robustness perspective, including deep SNN surveys and training methodologies \cite{RN272,RN273,RN279,RN280,diehl2015unsupervised}.
These approaches typically aim to mitigate or average out timing variability, but do not provide structural guarantees that the terminal output is invariant to the specific realization of spike
times.

\paragraph{Continuous-time neuron models and terminal output.}
Classical spiking neuron models, including integrate-and-fire, exponential integrate-and-fire, and related formulations, provide mechanistic rules for spike generation and admit continuous-time descriptions of neuronal dynamics \cite{RN271,RN276,brette2007exact}.
While these models capture rich transient behaviors, their terminal outputs are usually tied to specific reset rules, biophysical assumptions, or discretization schemes \cite{GerstnerKistler2002,DayanAbbott2005,Izhikevich2007}.
Therefore, existing analysis do not yield model-independent criteria for when terminal output is deterministic and insensitive to temporal stochasticity in spike timing.

\paragraph{Conservation- and constraint-based perspectives.}
Several works have explored conservation laws, bookkeeping principles, or constraint-based formulations as a means to understand neural computation from a global perspective.
Early examples include energy-based and attractor-style models, \eg, Hopfield networks, where computation is characterized by convergence to a terminal state determined by global constraints \cite{hopfield1982neural}.
Related constraint-based viewpoints also appear in energy-based learning frameworks (outside the spiking setting) \cite{RN278}.
In the neuromorphic and spiking context, conservation-inspired perspectives have been discussed as part of the broader effort to build interpretable and efficient event-driven systems \cite{boahen2017neuromorph,RN285}.
In particular, prior work introduced a discrete-time formulation of the Law of Charge Conservation (LoCC) for a specific spiking mechanism and demonstrated timing-invariant behavior in feedforward networks under that setting \cite{you2025vistream}.
The present work generalizes this line of research by developing a continuous-time formulation of LoCC that does not rely on a specific spiking mechanism, and by identifying minimal neuron-level and topological conditions under which conservation suffices to guarantee deterministic terminal computation.

\paragraph{QANN-SNN correspondence and terminal output computation.}
A closely related line of research studies the relationship between SNNs and ANNs, including ANN-to-SNN conversion, rate-based interpretations, and temporal or latency-based coding schemes \cite{diehl2015fast,RN289,RN291,RN292}.
These approaches often establish approximate correspondences that hold under long time horizons, specific coding assumptions, or bounded quantization errors, and are primarily concerned with matching firing rates or decision accuracy.
In contrast, our analysis yields an \emph{exact} correspondence between continuous-time spiking dynamics and static QANN computations at the terminal output, without relying on precise spike timing, rate averaging, or approximation arguments.

\paragraph{Network topology and determinism.}
The influence of network topology on neural dynamics has been widely studied,
particularly for recurrent architectures where complex and potentially non-convergent dynamical behaviors may arise \cite{RN294,RN282,RN293,RN283}.
Most existing analyses focus on expressivity, stability, or learning dynamics, rather than on topology-based guarantees of timing-insensitive terminal computation.
Our results offer a complementary perspective: deterministic terminal
computation is guaranteed for acyclic network topologies, even under
temporal variability.
By contrast, recurrent topologies can make the computation sensitive
to spike timing.

\section{Charge-Conserving Continuous-Time SNNs}


\label{sec:charge-conserving-snn}

Our primary motivation is to engineer algorithmic determinism into asynchronous systems that are inherently prone to temporal stochasticity. In standard spiking neuron models, the terminal output is often path-dependent, emerging as a byproduct of transient voltage dynamics that are sensitive to variable delays and jitter. To achieve robust determinism, we introduce a fundamental design shift: \textbf{moving from state-dependent dynamics to conservation-based bookkeeping}. 
The basic idea is to impose LoCC as a rigorous algorithmic constraint. By enforcing that every unit of aggregate input charge is strictly accounted for—regardless of the precise moment it is integrated or discharged—we structurally decouple the computational result from its temporal realization. This ensures that the system’s terminal output is determined solely by the total information processed, rendering it strictly invariant to the stochastic timing details of the continuous-time substrate.


\begin{figure}[!t]
  \centering
  \includegraphics[width=0.9\linewidth]{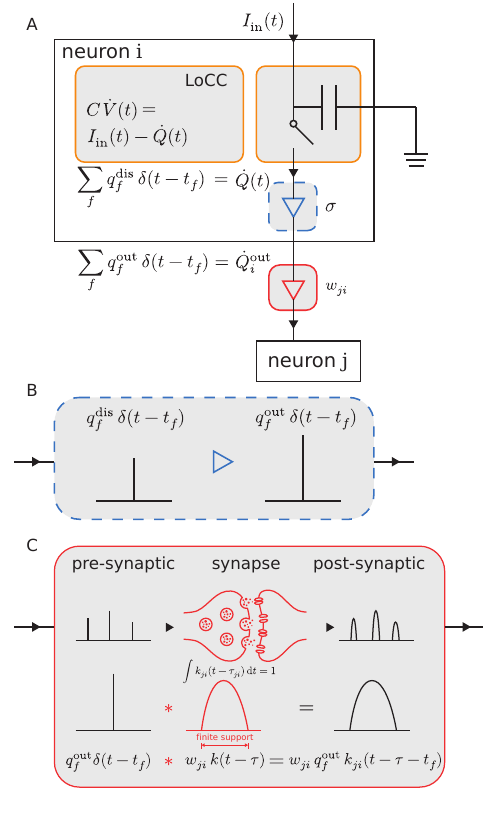}
  \caption{
  Charge-based abstraction of continuous-time spiking neurons.
   (\textbf{A}) Neuron-level dynamics.
  (\textbf{B}) Network interface.
    (\textbf{C}) Synaptic transmission.
  }
  \label{fig:diagram1}
\end{figure}

\subsection{Single Neuron Dynamics}
\label{subsec:neuron-level}

\cref{fig:diagram1}A illustrates neural dynamics are described in terms of membrane charge integration, spike-triggered discharge, and spike-wise output encoding.
This abstraction captures a broad class of continuous-time neuron models while remaining agnostic to the specific spike-generation mechanism.

\paragraph{Charge-conserving neuronal dynamics.}
Each neuron is characterized by a membrane potential $V(t)\in\mathbb{R}$ that
stores electric charge and by internal states governing spike generation.
Membrane dynamics can be expressed in charge form as
\begin{equation}
C\,\dot V(t) = I_{\mathrm{in}}(t) - \dot Q(t),
\label{eq:neuron-charge-dynamics}
\end{equation}
where $C$ denotes the membrane capacitance, $I_{\mathrm{in}}(t)$ is the total
input current, and $Q(t)$ is the cumulative discharged membrane charge induced by spikes. Equation~(\ref{eq:neuron-charge-dynamics}) enforces the conservation law of membrane charge, namely, \textbf{the law of charge conservation (LoCC)} (See Lemma~\ref{lem:single-neuron-locc} for details). It dictates that all input charge, if not converted to spikes with discharge, remains in the membrane potential.
Spike events occur at times $\{t_f\}$ determined by the neuron’s internal event
mechanism and induce instantaneous discharge increments, represented as
\begin{equation}
\dot Q(t)
=
\sum_f q_f^{\mathrm{dis}}\,\delta(t-t_f).
\end{equation}
Without loss of generality, we use zero initial conditions $V(0)=0$ and
$Q(0)=0$ by default throughout this work. Spike times may be subject to neuron-level timing constraints, such as refractory
periods or internal clocks.
When such constraints are present, we allow them to be stochastic.
Technical conventions for hybrid dynamics, including right-continuity and
measure-theoretic representations of spike-triggered discharge, are deferred to
Appendix~\ref{app:measure}.



\paragraph{Additional neuron-level design conditions.}
Charge conservation alone does not guarantee deterministic terminal behavior.
To obtain a well-defined, time-independent mapping from aggregate input charge to
a unique terminal state, the neuron model must satisfy two additional
\emph{design conditions}.

\emph{(i) Decoding structure at silence}
\label{cond:decoding-structure}.
For each admissible discharge level $Q\in\mathcal Q$, let
$\mathcal V^*(Q)\subset\mathbb R$ denote the associated silent set.
We require that the corresponding input-charge slices
\begin{equation}
\Omega_Q := \{\, Q + C V^* : V^* \in \mathcal V^*(Q) \,\}
\label{eq:input-charge-slice}
\end{equation}
form a partition of the real line.
This condition ensures that any total injected charge $z$ can be associated
with a discharge level $Q^*$ at a silent state.
Accordingly, this defines a decoding map
\[
\mathcal D:\mathbb R\to\mathcal Q,
\qquad
z\mapsto Q^* \;\;\text{such that}\;\; z\in\Omega_{Q^*}.
\]

\emph{(ii) Strict progress toward silence}
\label{cond:strict-progress}.
Define the distance-to-silence function
\[
\Phi(V,Q) := \mathrm{dist}\bigl(V,\mathcal V^*(Q)\bigr).
\]
The neuron is designed to satisfy a strict progress condition:
there exists $\delta>0$ such that, at each spike time $t_f$,
either the post-spike state is silent,
$V(t_f)\in\mathcal V^*(Q(t_f))$,
or
\begin{equation}
\Phi\bigl(V(t_f),Q(t_f)\bigr)
\;\le\;
\Phi\bigl(V(t_f^-),Q(t_f)\bigr)
-
\delta .
\label{eq:strict-progress}
\end{equation}
This condition enforces event-wise progress at spike times.
In particular, once external input ceases, spiking activity cannot persist
indefinitely.

Together, above conditions  ensure that every admissible input episode
converges to a silent state corresponding to a uniquely decoded terminal output (See Lemma~\ref{lem:single-neuron-locc} for details).
Formal definitions and technical elaborations are deferred to
Appendix~\ref{app:silence} and Appendix~\ref{app:finite-spiking}.

In addition to updating the internal discharge state, each spike-induced
discharge is re-encoded into an externally visible output charge via an
activation map $\sigma$,
\[
Q^{\mathrm{out}}(t) := \sigma\!\bigl(Q(t)\bigr).
\]
Accordingly, a spike at time $t_f$ with discharge $q_f^{\mathrm{dis}}$ produces
the output increment
\[
q_f^{\mathrm{out}}
=
\sigma\!\bigl(Q(t_f^-)+q_f^{\mathrm{dis}}\bigr)
-
\sigma\!\bigl(Q(t_f^-)\bigr).
\]
The externally visible output charge admits the representation
\[
Q^{\mathrm{out}}(t)
=
\sigma\!\bigl(Q(0)\bigr)
+
\sum_f q_f^{\mathrm{out}}\, H(t-t_f),
\]
where $H(\cdot)$ denotes the Heaviside step  function.
In particular, if $\sigma(0)\neq 0$, the initial discharge state contributes
a nonzero output offset prior to any spike event.

\subsection{Network Coupling via Charge-Preserving Transmission}
\label{subsec:network-coupling}
At the network level (Figure~1B--C), neurons interact only through spike-wise
output charge.

\paragraph{Spike-wise transmission interface.}
Neuron $j$ exposes the event stream $\{(t^j_f, q^{\mathrm{out}}_{j,f})\}$, where
$q^{\mathrm{out}}_{j,f}$ is the activation jump induced by the $f$-th discharge event;
all other discharge dynamics (e.g., $q^{\mathrm{dis}}_{j,f}$ and the evolution of $Q_j$ and $V_j$)
remain internal.

\paragraph{Charge-normalized synaptic filtering and bounded delays.}
When neuron $j$ emits a spike at time $t_f^j$, the associated output charge
$q_{j,f}^{\mathrm{out}}$ is transmitted to a postsynaptic neuron $i$ through
synaptic filtering and propagation delay.
Specifically, the spike contributes an atomic charge impulse that is convolved
with a synaptic kernel $k_{ij}(\cdot)$ and delayed by $\tau_{ij}\ge 0$.
The resulting synaptic input current to neuron $i$ is given by
\begin{equation}
{\small
\begin{aligned}
I_i^{\mathrm{syn}}(t)
=
\sum_{j=1}^N W_{ij}
\left(
\sum_f
q_{j,f}^{\mathrm{out}}\,
k_{ij}\!\left(t - t_f^j - \tau_{ij}\right)
+
q_0\,k_{ij}\!\left(t-\tau_{ij}\right)
\right),
\end{aligned}
}
\label{eq:network-synaptic-current}
\end{equation}
where $W_{ij}$ denotes the synaptic coupling weights and
$q_0:=\sigma(0)$ accounts for a possible baseline output charge.

The synaptic kernel is assumed to have finite support and to be
\emph{charge-normalized}, satisfying
\begin{equation}
\int_0^{\infty} k_{ij}(t)\,\mathrm{d}t = 1.
\label{eq:kernel-charge-normalization}
\end{equation}
Consequently, each spike transmits a total charge exactly equal to
$q_{j,f}^{\mathrm{out}}$, independent of the temporal shape of the kernel.

Propagation delays $\tau_{ij}$ model variability in synaptic transmission and
signal propagation, and may be stochastic.
Importantly, synaptic filtering and delays affect only the temporal realization
of spike interactions; the total transmitted charge remains invariant under LoCC.

\paragraph{External episode charge.}
Let $I^{\mathrm{ext}}_i(t)$ denote the external input current to neuron $i$ on $[0,T]$, and set
\[
u_i := \int_0^T I^{\mathrm{ext}}_i(t)\,\mathrm dt,\qquad u=(u_i)_{i=1}^n.
\]
We use $u$ to parameterize input episodes independently of spike timing.

\section{Deterministic Output under Temporal Stochasticity}
\label{sec:output-determinism}

In this section, we study how temporal stochasticity at the level of spike
generation, processing latency, and signal transmission affects the terminal output of spiking neural networks.
Rather than treating timing variability as a nuisance to be eliminated, we
ask under what structural conditions such variability becomes irrelevant to
the terminal computation.



We proceed in two steps.
We first establish determinism at the level of a single neuron, showing that
charge conservation together with additional neuron-level design conditions
induces a time-independent input--output map.
We then lift these results to the network level and show that acyclicity
suffices to collapse temporal variability into a deterministic terminal
outcome.

\subsection{Single-Neuron Determinism under Charge Conservation}
\label{subsec:single-neuron-determinism}

We first establish determinism at the level of a single spiking neuron.
Throughout this subsection, we consider a neuron whose dynamics satisfy the
charge-based formulation and neuron-level design conditions introduced in
Section~\ref{sec:charge-conserving-snn}.

\medskip

\begin{lemma}[LoCC for a Single Neuron]
\label{lem:single-neuron-locc}
Let $V(t)$ and $Q(t)$ evolve according to the charge-based membrane dynamics
Equation~(\ref{eq:neuron-charge-dynamics}), with spike-triggered discharge represented as an
atomic measure.
Then, for any $T \ge 0$, the membrane charge satisfies the identity
\begin{equation}
C\bigl(V(T)-V(0)\bigr)
=
Q^{\mathrm{in}}(0,T) - \bigl(Q(T)-Q(0)\bigr),
\label{eq:locc-single-neuron}
\end{equation}
where
\[
Q^{\mathrm{in}}(0,T) := \int_0^T I_{\mathrm{in}}(t)\,\mathrm{d}t
\]
denotes the cumulative injected input charge.
\end{lemma}

Lemma~\ref{lem:single-neuron-locc} expresses an exact bookkeeping identity at the
level of membrane charge: the change in stored charge equals the difference
between total injected charge and total spike-triggered discharge.
Importantly, this relation holds independently of the spike-generation
mechanism, spike timing, or internal state dynamics. A detailed proof is provided in Appendix~\ref{app:proof-locc-single}.

\medskip

\begin{lemma}[Uniqueness of discharge at silence under the decoding structure]
\label{lem:unique-silent-discharge}
Consider a neuron satisfying the decoding structure specified in
Section~\ref{sec:charge-conserving-snn}.
Then the aggregate injected charge uniquely determines the discharge level associated with a silent state.
\end{lemma}

Under the decoding structure at silence, the aggregate injected charge
uniquely determines the discharge level associated with a silent state,
independent of spike timing. A detailed proof is provided in Appendix~\ref{app:proof-unique-silent-discharge}.

\medskip

\begin{lemma}[Finite spiking under strict progress toward silence]
\label{lem:finite-spiking}
Consider a single neuron satisfying the design conditions of
Section~\ref{sec:charge-conserving-snn}, in particular
Equation~(\ref{eq:strict-progress}).
Assume that $I_{\mathrm{in}}(t)=0$ for all $t\ge T_{\mathrm{in}}$.
Then the neuron can emit only finitely many spikes for $t\ge T_{\mathrm{in}}$.
\end{lemma}

This lemma formalizes the role of the strict progress condition as a neuron-level design principle ensuring termination of spiking activity.
In particular, after the input episode ends, the neuron emits only finitely many spikes, although this number need not admit a uniform deterministic upper bound.
Temporal variability in spike timing or refractory effects may alter the transient trajectory, but cannot prevent convergence to silence.
A detailed proof is provided in Appendix~\ref{app:proof-finite-spiking}.

\medskip

\begin{theorem}[Time-Independent Input--Output Map of a Single Neuron]
\label{thm:single-neuron-determinism}
Consider a single neuron designed according to
Section~\ref{sec:charge-conserving-snn}, satisfying the charge conservation
principle, the decoding structure at silence, and the strict progress condition,
with reference initial conditions $V(0)=0$ and $Q(0)=0$.

Assume that the external input episode terminates at some finite time
$T_{\mathrm{in}}$, so that $I_{\mathrm{in}}(t)=0$ for all $t\ge T_{\mathrm{in}}$.
Then there exists a finite time $T_{\mathrm{term}}$ such that the externally
visible output
\[
Q^{\mathrm{out}}(t)=\sigma\!\bigl(Q(t)\bigr)
\]
is constant for all $t\ge T_{\mathrm{term}}$.
Moreover, the terminal output is given by the composition
\[
Q^{\mathrm{out}}_{\mathrm{term}}
=
(\sigma\circ\mathcal D)\!\left(Q^{\mathrm{in}}(0,T_{\mathrm{in}})\right),
\]
where $\mathcal D$ denotes the decoding map induced by the silent-set partition
Equation~(\ref{eq:input-charge-slice}).
\end{theorem}

\emph{Proof sketch.}
Finite termination follows from Lemma~\ref{lem:finite-spiking}.
Applying the single-neuron LoCC (Lemma~\ref{lem:single-neuron-locc}) at the
terminal time expresses the terminal discharge level in terms of the aggregate
input charge.
Uniqueness is then guaranteed by the decoding structure
(Lemma~\ref{lem:unique-silent-discharge}), and the output follows by
application of the decoding map $\sigma$.
A complete proof is provided in Appendix~\ref{app:proof-single-neuron-determinism}.

\subsection{Network-Level Determinism under Acyclic Coupling}
\label{subsec:network-determinism}

We now extend single-neuron determinism to the network level, focusing on acyclic interaction graphs. We show that for these topologies, the terminal decoded output remains strictly invariant to temporal stochasticity—including variability in input delivery, spike emission, and transmission delays—allowing the network to be evaluated sequentially along its topological order.


\begin{theorem}[Acyclic networks admit a unique timing-invariant terminal output]
\label{thm:acyclic-network-determinism}
Consider an event-driven spiking neural network with synaptic weight matrix $W$
whose interaction graph is acyclic.
Assume each neuron satisfies Theorem~\ref{thm:single-neuron-determinism}.
Fix an external input episode and let $u\in\mathbb R^n$ be the corresponding
external aggregate charge vector.

Then the network terminates after finitely many spike events and admits a unique terminal
decoded output $\kappa^\ast\in\mathbb R^n$ satisfying
\begin{equation}
\label{eq:terminal-fixed-point}
\kappa^\ast \;=\; (\sigma\circ\mathcal D)\!\left(u + W\,\kappa^\ast\right),
\end{equation}
where $(\sigma\circ\mathcal D)$ is applied componentwise.
\end{theorem}

\noindent Equivalently, if $z:=u+W\kappa^\ast$ denotes the total inflow charge vector at termination,
then $\kappa^\ast=(\sigma\circ\mathcal D)(z)$.

\emph{Proof sketch.}
Acyclicity yields a topological ordering in which interactions propagate strictly forward.
Along this order, each neuron's aggregate input charge is uniquely determined by the external input and upstream terminal outputs.
Single-neuron determinism (Theorem~\ref{thm:single-neuron-determinism}) then fixes the terminal decoded output uniquely and independently of timing.
Induction over the ordering gives a unique timing-invariant terminal network output.
A complete proof is provided in Appendix~\ref{app:proof-acyclic}.

\begin{remark}[ANN representation of the terminal computation]
\label{rem:ann-repr}
Equation~(\ref{eq:terminal-fixed-point}) is exactly the input--output relation of a ANN
with weight matrix $W$ and activation given by the single-neuron terminal map $(\sigma\circ\mathcal D)$
from Theorem~\ref{thm:single-neuron-determinism}.
In particular, no quantization is assumed at this level; quantized cases arise under specific realizations of $\mathcal D$ (see Section~\ref{sec:realization-and-correspondence}).
\end{remark}

\begin{remark}[Finite-time termination without uniform bounds]
\label{rem:no-uniform-bounds}
Once the external input episode has ended, strict progress guarantees that spiking activity terminates in finite time.
However, even for a single neuron, neither the number of post-input spikes nor the termination time admits a uniform deterministic upper bound
that depends only on the total injected charge and neuron-level constants:
different temporal realizations of the same aggregate input can drive the state arbitrarily far from silence,
leading to arbitrarily many spikes after input cessation.
Consequently, network-level convergence-time bounds that depend only on coarse graph parameters (e.g., depth) and local timing constants
are not valid in general.
\end{remark}

\subsection{Implications, Counterexamples, and the Role of Acyclicity}
\label{subsec:determinism-discussion}

The results of Section~\ref{subsec:network-determinism} establish that, for
acyclic network architectures composed of neurons satisfying the design
conditions of Section~\ref{sec:charge-conserving-snn}, temporal stochasticity
affects only transient trajectories and not the terminal decoded output.
We now clarify the scope of this invariance and show, by a minimal
counterexample, that it does not extend beyond acyclic interaction graphs.

\paragraph{Terminal invariance under temporal perturbations.}
Under the conditions of Theorem~\ref{thm:acyclic-network-determinism}, the terminal
decoded output depends only on the aggregate injected charge and is therefore
invariant under a broad class of temporal perturbations.
These include reordering of input delivery in time, neuron-level processing
latencies, synaptic transmission delays, and temporary spike-suppression
mechanisms such as refractory or inhibition windows, provided that the aggregate
injected input remains unchanged.
Temporal variability may substantially alter intermediate spike patterns, but
leaves no footprint on the terminal decoded state.

\paragraph{A minimal counterexample beyond acyclicity.}
Once directed cycles are introduced into the interaction graph, the inductive
charge-propagation argument in Theorem~\ref{thm:acyclic-network-determinism} no longer applies.
Even if every neuron satisfies the same single-neuron design conditions, recurrent coupling can
break terminal determinism.

\begin{example}[Static fixed points need not be dynamically reachable]
\label{ex:unreachable}
Let $\kappa\in\mathcal K^3$ denote the decoded terminal outputs.
Consider the fixed-point equation
\[
\kappa=\sigma_q(W\kappa+b),
\qquad
\sigma_q(x):=\lfloor [x]_+ \rfloor \ \text{(componentwise)},
\]
with
\[
W=
\begin{pmatrix}
0 & 0 & 0\\
1 & 0 & -2\\
1 & -2 & 0
\end{pmatrix},
\qquad
b=
\begin{pmatrix}
1\\0\\0
\end{pmatrix}.
\]
This network contains a directed cycle between neurons $2$ and $3$.
The equation admits multiple solutions, e.g.,
\[
\kappa=(1,1,0)^\top
\quad\text{and}\quad
\kappa=(1,0,1)^\top.
\]
However, the synchronous iteration
\[
\kappa^{(t+1)}=\sigma_q(W\kappa^{(t)}+b),
\qquad
\kappa^{(0)}=(0,0,0)^\top,
\]
produces a 2-cycle $(1,0,0)^\top \leftrightarrow (1,1,1)^\top$ and hence does not converge to any fixed point.
\end{example}

\paragraph{Acyclicity as a sharp structural boundary.}
Example~\ref{ex:unreachable} shows that acyclicity is not a proof artifact but a structural boundary:
it is precisely the condition under which charge bookkeeping, well-posed decoding, and finite spiking
collapse temporal dynamics into a unique terminal output.
With directed cycles, charge conservation alone no longer ensures determinism; additional stabilization
(e.g., damping, scheduling, or update constraints) is generally required.
Further concrete counterexamples, including single-neuron self-loops and timing-sensitive feedback,
are collected in Appendix~\ref{app:counterexamples}.
\section{Realization of Charge-Conserving SNNs and Correspondence to QANNs}
\label{sec:realization-and-correspondence}

We now present a concrete event-driven spiking neuron that realizes the charge-based design principles introduced in Section~\ref{sec:charge-conserving-snn}.
The purpose of this construction is not to propose a new neuron model, but to demonstrate that the abstract framework admits a simple and internally consistent realization.
The resulting dynamics can be viewed as a generalized ST-BIF neuron.


\subsection{Realization of charge-conserving neurons}
\label{subsec:neuron-realization}
As a concrete realization, we specify an event-driven mechanism that implements the decoding map $\mathcal D$
in Section~\ref{sec:charge-conserving-snn} via spike-triggered transitions between discharge levels,
while preserving charge conservation and guaranteeing termination.

Consider a neuron whose membrane potential $V(t)$ and discharge variable $Q(t)$ are related by the charge-conservation identity
\[
C\,\dot V(t)=I_{\mathrm{in}}(t)-\dot Q(t),
\]
which enforces the Law of Charge Conservation at the neuron level.

\emph{(i) Decoding structure.}
To realize the finite decoding required in Section~\ref{sec:charge-conserving-snn}, we restrict the discharged membrane charge to a finite ordered set
\[
\mathcal Q=\{Q_m,Q_{m+1},\ldots,Q_M\}, \qquad m\le 0\le M,
\]
so that the terminal decoding takes values in $\mathcal Q$.
For $q\in\mathcal Q$, define its successor and predecessor by
\[
\begin{aligned}
\operatorname{succ}(q)&:=\min\{q'\in\mathcal Q:\ q'>q\},\\
\operatorname{pred}(q)&:=\max\{q'\in\mathcal Q:\ q'<q\},
\end{aligned}
\]
whenever they exist.

\emph{(ii) Strict progress via event-driven discharge updates.}
To accommodate impulsive input currents while preserving charge bookkeeping, we introduce the pre-decision potential
\[
\hat V(t):=V(t^-)+\frac{Q^{\mathrm{in}}(t)-Q^{\mathrm{in}}(t^-)}{C},
\]
which represents the membrane potential immediately after integrating any instantaneous input charge and before making a spike decision.
Whenever the refractory condition permits spike generation, a spike is emitted purely as an event mechanism to enforce transitions between neighboring discharge levels:
\begin{equation}
\begin{cases}
\hat V(t)\ge V_{\mathrm{thr}}(Q(t^-)) \ \text{and}\ Q(t^-)<Q_M,
& \text{upward spike},\\[4pt]
\hat V(t)<0 \ \text{and}\ Q(t^-)>Q_m,
& \text{downward spike},
\end{cases}
\label{eq:region_spike_rule}
\end{equation}
where
\[
V_{\mathrm{thr}}(q):=\dfrac{\operatorname{succ}(q)-q}{C}, \qquad q<Q_M.
\]
At a spike time $t_f$, define the discharge increment
\[
q_f^{\mathrm{dis}}=
\begin{cases}
\operatorname{succ}(Q(t_f^-))-Q(t_f^-), & \text{upward spike},\\[4pt]
\operatorname{pred}(Q(t_f^-))-Q(t_f^-), & \text{downward spike},
\end{cases}
\]
and update
\[
Q(t_f)=Q(t_f^-)+q_f^{\mathrm{dis}},\qquad
V(t_f)=\hat V(t_f)-\dfrac{q_f^{\mathrm{dis}}}{C}.
\]
By construction $Q(t_f)\in\mathcal Q$, and each spike advances $Q$ by one step in $\mathcal Q$ toward a terminal discharge level once the input episode ends,
thereby realizing strict progress and ensuring a well-defined terminal decoded output.
A formal verification that this realization satisfies the design conditions of Section~\ref{sec:charge-conserving-snn}
is provided in Appendix~\ref{app:verification}.

\begin{remark}[Interpretation as a generalized ST-BIF neuron]
The event-driven mechanism above can be viewed as a generalized spike-triggered balanced integrate-and-fire (ST-BIF) neuron:
the membrane integrates charge according to the conservation law, while spikes serve solely as event triggers that enforce discrete discharge transitions between neighboring levels in $\mathcal Q$.
This interpretation connects the present realization to the ST-BIF formulation used in our earlier work~\cite{you2025vistream},
but here it is derived as a simple implementation of our decoding and strict-progress design conditions.
\end{remark}

\begin{figure}[!t]
  \centering
  \includegraphics[width=0.85\linewidth]{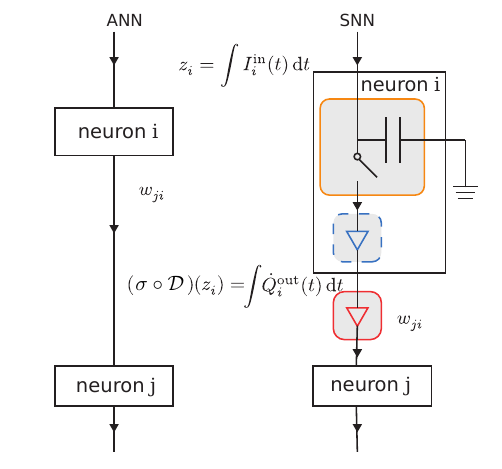}
  \caption{
  Structural correspondence between charge-based spiking neurons and artificial
  neurons.
  }
  \label{fig:diagram2}
\end{figure}

\subsection{SNN terminal outputs to QANN representations}
\label{subsec:snn-to-ann}

In this subsection, we prove that output of a Charge-conserving SNN admit a QANN representation using the above realization.
As before, this correspondence is purely representational: it depends only on the terminal input--output behavior and abstracts away spike timing and transients.

\begin{theorem}[Terminal output admits a QANN representation]
\label{thm:snn_to_qann}
Consider an acyclic spiking neural network whose neurons follow the model in
Section~\ref{subsec:neuron-realization}, with synaptic weight matrix $W$.
Fix an external input episode and let $u\in\mathbb R^n$ denote its external aggregate charge vector.
Then the terminal output vector $\kappa^\ast$ equals the output of the acyclic network
with weights $W$ and componentwise activation $\sigma\circ\mathcal D$, evaluated at input $u$.
\end{theorem}

\noindent A proof is given in Appendix~\ref{app:proof-SNN2ANN}.
Under the realization in Section~\ref{subsec:neuron-realization}, $\mathcal D$ is discrete-valued,
so the induced network is a QANN.

\begin{remark}[Representation vs.\ dynamics]
The correspondence concerns only the terminal input--output map and ignores spike timing.
Any implementation realizing the same terminal map $\sigma\circ\mathcal D$ shares this representation.
\end{remark}

\subsection{Realizing QANNs as SNNS}
\label{subsec:ann-to-snn}

We now consider the reverse direction: realizing a acyclic QANN with quantized activations as an event-driven spiking neural network.
Here the QANN specifies only a static input--output map, so an explicit spiking dynamics (e.g., the realization in Section~\ref{subsec:neuron-realization}) must be constructed.

We focus on quantized activations induced by a discrete decoding map.
Specifically, we assume each activation is of the form $(\sigma\circ\mathcal D)$, where
$\mathcal D:\mathbb R\to\mathcal Q$ takes values in a finite set $\mathcal Q$ (as in Section~\ref{subsec:neuron-realization}).
Equivalently, $\mathcal D$ is piecewise constant on a partition of $\mathbb R$ into disjoint half-open intervals,
and thus $(\sigma\circ\mathcal D)$ is a step-like (quantized) nonlinearity.

\begin{theorem}[QANN-to-SNN realization for quantized activations]
\label{thm:ann_to_snn}
Consider a acyclic QANN whose activation functions are of the form $(\sigma\circ\mathcal D)$,
where $\mathcal D$ is a discrete-valued decoding map that is piecewise constant on a partition of $\mathbb R$
into disjoint half-open intervals (as described above).
Then there exists an event-driven spiking neural network composed of the generalized ST-BIF neurons in
Section~\ref{subsec:neuron-realization} such that, for any fixed input episode,
the terminal decoded output of the SNN coincides with the QANN output.
\end{theorem}
A detailed proof is provided in Appendix~\ref{app:proof-ANN2SNN}.

\begin{remark}[Scope]
The restriction to quantized activations reflects the discrete decoding structure of the realized spiking neurons.
More general activation functions may be approximated by refining the decoding partition (equivalently, enlarging the discharge set $\mathcal Q$),
but approximation guarantees are beyond the scope of the present work.
\end{remark}

\section{Conclusion}

In this work, we addressed the fundamental challenge of designing reliable neuromorphic systems using inherently stochastic continuous-time hardware. We developed a unified framework that establishes LoCC as a rigorous algorithmic constraint which, when coupled with minimal neuron-level design conditions and acyclic connectivity, guarantees deterministic network computation. This structural determinism establishes an exact representational correspondence between event-driven SNNs and quantized ANNs. This result fundamentally reframes the relationship between the two paradigms: rather than viewing SNNs as noisy approximations of formal neural networks, we demonstrate that they can serve as precise, energy-efficient dynamic realizations of static algorithms. This perspective paves the way for a unified design methodology that seamlessly translates the algorithmic power of deep learning onto asynchronous neuromorphic substrates without approximation errors. Future work will explore the stability properties of charge-conserving recurrent networks and leverage these conservation constraints to guide biologically plausible learning algorithms.

\section*{Acknowledgments}
This work is partially supported by Shanghai Artificial Intelligence Laboratory, the National Key R\&D Program of China under Grant No.~2022YFB4500200 (Z.H.) and No.~2022YFA1008200 (Y.Z.), the National Natural Science Foundation of China under Grant No.~62102257 (Z.H.) and No.~12571567 (Y.Z.) and the Natural Science Foundation of Shanghai under Grant No.~25ZR1402280 (Y.Z.).

\bibliographystyle{icml2026}
\bibliography{bibliography}

%

\newpage
\appendix
\onecolumn

\appendix

\section{Technical Conventions and Measure-Theoretic Formulation}
\label{app:measure}

All state variables are taken to be right-continuous with left limits.
For a spike time $t_k$, we write
\[
X(t_k^-) := \lim_{t\uparrow t_k} X(t), \qquad
\Delta X(t_k) := X(t_k) - X(t_k^-).
\]

Spike-triggered discharge is represented as an atomic measure
\[
dQ(t) := \sum_f q_f \delta(t-t_f)\,dt,
\]
so that $Q(T)-Q(0)=\int_{(0,T]} dQ(t)$.
with jumps occurring at spike times $\{t_f\}$.

Throughout, integrals are taken over right-closed intervals $(0,T]$,
consistent with the càdlàg convention.
In particular, jump contributions at time $t_f$ are included in
$Q(t_f)$ but not in $Q(t_f^-)$.

Equivalently, cumulative spike effects may be represented using the step
function $H$, defined by $H(t)=0$ for $t<0$ and $H(t)=1$ for $t\ge 0$, so that
\[
Q(t)=Q(0)+\sum_f q_f\, H(t-t_f).
\]

\section{Silent Sets and Decoding Structure}
\label{app:silence}

For each discharge level $Q\in\mathcal Q$, recall the definition
\[
\Omega_Q := \{ Q + C V^* : V^* \in \mathcal V^*(Q) \}.
\]

\begin{definition}[Partition of input-charge space]
\label{def:partition}
The family $\{\Omega_Q\}_{Q\in\mathcal Q}$ is said to form a partition of
$\mathbb R$ if it satisfies:
\begin{enumerate}
  \item \textbf{Covering:} $\bigcup_{Q\in\mathcal Q} \Omega_Q = \mathbb R$.
  \item \textbf{Disjointness:} $\Omega_Q \cap \Omega_{Q'} = \varnothing$
  whenever $Q \neq Q'$.
\end{enumerate}
\end{definition}

This structure ensures that the total injected charge uniquely determines
the discharge level at any silent steady state.

\section{Finite Spiking and Eventual Rest}
\label{app:finite-spiking}

Define the distance-to-silence function
\[
\Phi(t) := \mathrm{dist}\bigl(V(t), \mathcal V^*(Q(t))\bigr).
\]
The neuron model is designed to satisfy a \emph{strict progress condition}:
there exists a constant $\delta>0$ such that, at each spike time $t_f$, either
the post-spike state is silent,
\[
V(t_f)\in\mathcal V^*(Q(t_f)),
\]
or the distance to silence strictly decreases,
\[
\Phi(t_f) \le \Phi(t_f^-) - \delta .
\]
Once a silent state is reached, no further spikes are admissible unless the
state leaves the silent set due to external input.

\section{Additional Proofs}
\label{app:additional-proofs}

\subsection{Proof of Lemma~\ref{lem:single-neuron-locc}}
\label{app:proof-locc-single}
\begin{proof}
The subthreshold membrane dynamics satisfy
\[
C\,\dot V(t)=I_{\mathrm{in}}(t)-\dot Q(t),
\]
where $\dot Q(t)$ is a sum of Dirac impulses corresponding to spike-triggered discharge events.
Integrating over $[0,T]$ yields
\[
C\bigl(V(T)-V(0)\bigr)
=
\int_0^T I_{\mathrm{in}}(t)\,\mathrm{d}t
-
\bigl(Q(T)-Q(0)\bigr),
\]
which is exactly Equation~(\ref{eq:locc-single-neuron}).
\end{proof}

\subsection{Proof of Lemma~\ref{lem:unique-silent-discharge}}
\label{app:proof-unique-silent-discharge}
\begin{proof}
Let $z\in\mathbb R$ denote the total injected input charge.
Since $\{\Omega_q\}_{q\in\mathcal Q}$ forms a disjoint partition of $\mathbb R$,
there exists a unique $q^\ast\in\mathcal Q$ such that $z\in\Omega_{q^\ast}$.

Consider any trajectory that reaches an eventual-rest state $(V^\ast,Q^\ast)$ under total input charge $z$.
Applying the Law of Charge Conservation over $[0,\infty)$ gives
\[
z=Q^\ast+CV^\ast.
\]
Since $(V^\ast,Q^\ast)$ is silent, we have $V^\ast\in\mathcal V^\ast(Q^\ast)$, hence
$z\in\Omega_{Q^\ast}$ by definition of $\Omega_{Q^\ast}$.
By uniqueness of the partition, $Q^\ast=q^\ast$.
Therefore the terminal discharge level, and hence the decoded output $\sigma(Q^\ast)$, is uniquely determined by $z$.
\end{proof}

\subsection{Proof of Lemma~\ref{lem:finite-spiking}}
\label{app:proof-finite-spiking}
\begin{proof}
Let $\{t_k\}_{k\ge1}$ denote the spike times after the input has ended.
If for some $k$ the post-spike state satisfies $V(t_k)\in\mathcal V^\ast(Q(t_k))$,
then by the trigger rule no further spikes are admissible and the proof is complete.

Otherwise, by the strict progress condition Equation~(\ref{eq:strict-progress}), every spike yields a uniform decrease
\[
\Phi(t_k)\le \Phi(t_k^-)-\delta
\qquad\text{for some }\delta>0.
\]
Iterating gives $\Phi(t_k)\le \Phi(t_1^-)-(k-1)\delta$.
Since $\Phi\ge 0$, this cannot hold for arbitrarily large $k$, hence only finitely many spikes occur and the trajectory reaches eventual rest.
\end{proof}

\subsection{Proof of Theorem~\ref{thm:single-neuron-determinism}}
\label{app:proof-single-neuron-determinism}
\begin{proof}
By Lemma~\ref{lem:finite-spiking}, the neuron emits only finitely many spikes after the input has ended.
Let $t_{\mathrm{term}}$ be the time of the last spike and set $Q^{\mathrm{term}}:=Q(t_{\mathrm{term}})$.

Applying the single-neuron Law of Charge Conservation (Lemma~\ref{lem:single-neuron-locc}) over $[0,t_{\mathrm{term}}]$ gives
\[
Q^{\mathrm{in}}(0,t_{\mathrm{term}})
=
Q^{\mathrm{term}}+C V(t_{\mathrm{term}}).
\]
Since the post-spike state is silent, $V(t_{\mathrm{term}})\in\mathcal V^\ast(Q^{\mathrm{term}})$.
By Lemma~\ref{lem:unique-silent-discharge}, the terminal discharge level $Q^{\mathrm{term}}$ is uniquely determined by the aggregate injected charge
$Q^{\mathrm{in}}(0,t_{\mathrm{term}})$ (together with the fixed initial state).
Finally, the terminal decoded output equals $\sigma(Q^{\mathrm{term}})$, proving that the terminal output depends only on the aggregate input and is independent of its temporal realization.
\end{proof}

\subsection{Proof of Theorem~\ref{thm:acyclic-network-determinism}}
\label{app:proof-acyclic}
\begin{proof}
Fix an acyclic interaction graph and a topological ordering of neurons.
Fix an external input episode and let $u\in\mathbb R^n$ denote its external aggregate charge vector.
We argue by induction along the ordering.

For the first neuron, the total inflow charge is $z_1=u_1$. By
Theorem~\ref{thm:single-neuron-determinism},
\[
\kappa_1^\ast = (\sigma_1\circ\mathcal D_1)(z_1) = (\sigma_1\circ\mathcal D_1)(u_1).
\]

Assume $\kappa_1^\ast,\dots,\kappa_{i-1}^\ast$ have been uniquely determined.
Acyclicity implies that the total inflow charge at neuron $i$ is
\[
z_i \;=\; u_i + \sum_{j<i} w_{ij}\,\kappa_j^\ast .
\]
Applying Theorem~\ref{thm:single-neuron-determinism} yields
\[
\kappa_i^\ast = (\sigma_i\circ\mathcal D_i)(z_i)
= (\sigma_i\circ\mathcal D_i)\!\left(u_i+\sum_{j<i} w_{ij}\,\kappa_j^\ast\right).
\]

Proceeding to $i=n$ yields a unique terminal output vector $\kappa^\ast$.
Collecting the nodewise relations gives the componentwise fixed-point equation
\[
\kappa^\ast = (\sigma\circ\mathcal D)(u+W\kappa^\ast),
\]
which is \eqref{eq:terminal-fixed-point}.
\end{proof}
\section{Counterexamples}
\label{app:counterexamples}

This appendix provides concrete counterexamples showing that acyclicity in
Theorem~\ref{thm:acyclic-network-determinism} is a sharp boundary for deterministic terminal behavior.

Throughout, we use the quantized ReLU
\[
\sigma_q(x):=\bigl\lfloor [x]_+ \bigr\rfloor
\quad\text{(applied componentwise)}.
\]

\begin{example}[Positive self-loop: multiple fixed points with finite termination]
\label{ex:single-positive}
Consider the scalar fixed-point equation
\[
\kappa=\sigma_q(0.5\,\kappa+0.6),
\qquad \kappa\in\mathbb Z_{\ge 0}.
\]
Then $\kappa=0$ and $\kappa=1$ are both fixed points:
\[
\sigma_q(0.6)=0,\qquad \sigma_q(1.1)=1,
\]
and no other fixed points exist since $\sigma_q(0.5\kappa+0.6)\le 1$ for all $\kappa\ge 0$.
This captures a minimal multistability mechanism induced by positive feedback: termination may still hold,
but the terminal decoded output need not be unique.
\end{example}

\begin{example}[Negative self-loop: a 2-cycle under the decoded update]
\label{ex:single-negative}
Consider the synchronous decoded update
\[
\kappa^{(t+1)}=\sigma_q(-0.5\,\kappa^{(t)}+1.2),
\qquad \kappa^{(t)}\in\mathbb Z_{\ge 0}.
\]
Starting from $\kappa^{(0)}=0$ yields
\[
\kappa^{(1)}=\sigma_q(1.2)=1,\qquad
\kappa^{(2)}=\sigma_q(0.7)=0,
\]
hence a 2-cycle $0\leftrightarrow 1$.
This illustrates that feedback with inhibition can prevent convergence of the decoded-level dynamics,
so a time-independent terminal output may fail to exist.
\end{example}

\begin{example}[Reachability gap in a cyclic network]
\label{ex:reachability-gap}
Consider the 3-neuron fixed-point equation $\kappa=\sigma_q(W\kappa+b)$ with
\[
W=
\begin{pmatrix}
0 & 0 & 0\\
1 & 0 & -2\\
1 & -2 & 0
\end{pmatrix},
\qquad
b=
\begin{pmatrix}
1\\0\\0
\end{pmatrix},
\qquad
\kappa\in\mathbb Z_{\ge 0}^3.
\]
The directed cycle between neurons $2$ and $3$ yields multiple fixed points, e.g.,
\[
(1,1,0)^\top,\qquad (1,0,1)^\top,
\]
since
\[
W(1,1,0)^\top+b=(1,1,-1)^\top \mapsto (1,1,0)^\top,\quad
W(1,0,1)^\top+b=(1,-1,1)^\top \mapsto (1,0,1)^\top.
\]
However, the synchronous iteration from $\kappa^{(0)}=(0,0,0)^\top$ produces
\[
\kappa^{(1)}=(1,0,0)^\top,\qquad
\kappa^{(2)}=(1,1,1)^\top,\qquad
\kappa^{(3)}=(1,0,0)^\top,
\]
hence a 2-cycle $(1,0,0)^\top \leftrightarrow (1,1,1)^\top$ and does not converge to any fixed point.
This demonstrates a reachability gap: static solutions exist, but the induced decoded dynamics can fail to reach them.
\end{example}

\section{Verification of the neuron dynamics fit the conditions}
\label{app:verification}

We verify that the generalized ST-BIF realization in
Section~\ref{subsec:neuron-realization} satisfies the neuron-level design
conditions imposed in Section~\ref{sec:charge-conserving-snn}, in particular:
(i) a finite decoding structure given by a partition of input-charge slices, and
(ii) strict progress toward silence after the input episode ends (hence finite spiking and a well-defined terminal decoded output).

\emph{(i) Partition of admissible steady states.}
At an eventual-rest steady state, the no-spike condition induced by Equation~(\ref{eq:region_spike_rule}) implies:

\smallskip
\noindent
\emph{Interior levels $Q_m<q<Q_M$.}
Downward spikes are enabled, hence $V^*\ge 0$; upward spikes are prevented only if
$V^*<(\operatorname{succ}(q)-q)/C$. Therefore
\[
\Omega_q=\{\,q+C V^*:\ V^*\in[0,(\operatorname{succ}(q)-q)/C)\,\}=[q,\operatorname{succ}(q)).
\]

\smallskip
\noindent
\emph{Minimal level $q=Q_m$.}
Downward spikes are disabled, so no-spike only requires
$V^*<(\operatorname{succ}(Q_m)-Q_m)/C$, yielding
\[
\Omega_{Q_m}=(-\infty,\operatorname{succ}(Q_m)).
\]

\smallskip
\noindent
\emph{Maximal level $q=Q_M$.}
Upward spikes are disabled, and no-spike requires $V^*\ge 0$, hence
\[
\Omega_{Q_M}=[Q_M,+\infty).
\]

Therefore, the family $\{\Omega_q\}_{q\in\mathcal Q}$ forms a disjoint partition of $\mathbb R$.

\emph{(ii) Strict progress and finite spiking.}
For each discharge level $q\in\mathcal Q$, define the admissible steady-state set
\[
\mathcal V^*(q)=
\begin{cases}
(-\infty,\,V_{\mathrm{thr}}(Q_m)), & q=Q_m,\\
[0,\,V_{\mathrm{thr}}(q)), & Q_m<q<Q_M,\\
[0,+\infty), & q=Q_M.
\end{cases}
\]
Introduce the Lyapunov function
\[
\Phi(t):=\mathrm{dist}\bigl(V(t),\mathcal V^*(Q(t))\bigr).
\]
Let $t_f$ be a spike time. The spike-triggered update in
Section~\ref{subsec:neuron-realization} maps the pre-decision potential $\hat V(t_f)$ to
\[
V(t_f)=\hat V(t_f)-\frac{q_f^{\mathrm{dis}}}{C},
\qquad
Q(t_f)=Q(t_f^-)+q_f^{\mathrm{dis}}\in\mathcal Q,
\]
and $Q$ changes by exactly one step in $\mathcal Q$ (via $\operatorname{succ}$ or $\operatorname{pred}$)
whenever the corresponding direction is admissible.

If $(V(t_f),Q(t_f))\in\mathcal T$, then by the trigger rule Equation~(\ref{eq:region_spike_rule}) no further spikes are possible.
Otherwise, the spike produces a strict decrease in $\Phi$.
Since $\mathcal Q$ is finite, define the minimal step size
\[
\delta :=
\min_{q\in\mathcal Q:\ \operatorname{succ}(q)\ \text{and}\ \operatorname{pred}(q)\ \text{exist}}
\bigl(\operatorname{succ}(q)-q,\; q-\operatorname{pred}(q)\bigr)
>0.
\]
A direct case check on upward vs.\ downward spikes shows that every non-terminal spike satisfies
\[
\Phi(t_f)\le \Phi(t_f^-)-\delta.
\]
Hence, once the external input episode has ended (so that $\hat V$ changes only through the neuron dynamics),
only finitely many non-terminal spikes can occur, after which the trajectory enters $\mathcal T$ and remains silent.
Therefore the neuron emits finitely many spikes and reaches a terminal discharge level, yielding a well-defined terminal decoded output.

\section{SNN--QANN correspondence under the generalized ST-BIF realization}
\label{app:transformation}

\subsection{Proof of Theorem~\ref{thm:snn_to_qann} (SNN $\to$ QANN)}
\label{app:proof-SNN2ANN}
\begin{proof}
Consider an acyclic SNN composed of the generalized ST-BIF neurons realized in
Section~\ref{subsec:neuron-realization}. Fix an external input episode and let
$u\in\mathbb R^n$ denote the corresponding external aggregate charge vector.
By Appendix~\ref{app:verification}, each neuron reaches a terminal discharge level and thus admits
a well-defined terminal decoded output. Denote the single-neuron terminal map by
\[
\kappa_i^\ast = (\sigma_i\circ \mathcal D_i)(z_i),
\]
where $z_i$ is the total aggregate inflow charge received by neuron $i$ over the episode.

Since the interaction graph is acyclic, fix a topological order. By
Theorem~\ref{thm:acyclic-network-determinism}, the inflow charges satisfy
\[
z \;=\; u + W\,\kappa^\ast,
\]
hence for each node $i$,
\[
\kappa_i^\ast
=
(\sigma_i\circ \mathcal D_i)\!\left(u_i+\sum_{j<i} w_{ij}\,\kappa_j^\ast\right).
\]
Collecting these nodewise relations yields an acyclic network with weights $W$ and componentwise
activation $\sigma_i\circ\mathcal D_i$ whose output coincides with the terminal decoded output
$\kappa^\ast$ of the SNN. Under the realization in Section~\ref{subsec:neuron-realization},
$\mathcal D_i$ is discrete-valued, so this network is a QANN.
\end{proof}

\subsection{Proof of Theorem~\ref{thm:ann_to_snn} (QANN $\to$ SNN)}
\label{app:proof-ANN2SNN}

\begin{proof}
Let the QANN be defined on an acyclic directed graph and suppose each activation is quantized as
\[
\phi_i=\sigma_i\circ\mathcal D_i,
\]
with $\mathcal D_i$ discrete-valued and piecewise constant on a partition of $\mathbb R$
(Section~\ref{subsec:ann-to-snn}).

Construct an SNN on the same acyclic graph by assigning to each node $i$ the generalized ST-BIF neuron of
Section~\ref{subsec:neuron-realization} with admissible discharge set $\mathcal Q_i$ and readout $\sigma_i$,
and choosing its decoding so that the induced terminal map equals $\phi_i$.
By Appendix~\ref{app:verification}, for any total inflow charge $z_i$ this neuron reaches a terminal discharge level
$Q_i^\ast=\mathcal D_i(z_i)$ and hence outputs
\[
\kappa_i^\ast=\sigma_i(Q_i^\ast)=(\sigma_i\circ\mathcal D_i)(z_i)=\phi_i(z_i).
\]

Since the graph is acyclic, evaluate nodes along a topological order.
For an external episode with aggregate charge vector $u$, the inflow charges satisfy
$z=u+W\kappa^\ast$ (Theorem~\ref{thm:acyclic-network-determinism}), i.e.,
\[
z_i=u_i+\sum_{j<i} w_{ij}\,\kappa_j^\ast.
\]
Therefore the constructed neuron produces $\kappa_i^\ast=\phi_i(z_i)$, matching the QANN computation at node $i$.
Induction over the topological order yields coincidence of the SNN terminal decoded output with the QANN output.
\end{proof}




\end{document}